%% file: main.tex
\documentclass{article}
\usepackage{codefuse_tech_report}
\usepackage[colorlinks = true,
            linkcolor = blue,
            urlcolor  = blue,
            citecolor = blue,
            anchorcolor = blue]{hyperref}   
\usepackage{microtype}
\usepackage{hyperref}
\usepackage{xurl}
\usepackage{booktabs}
\usepackage{enumitem}
\usepackage{nicefrac}       
\usepackage{xcolor}         
\usepackage{colortbl}
\usepackage{latexsym}
\usepackage{soul}
\usepackage{microtype}
\usepackage{subcaption}
\usepackage{bbding}
\usepackage{longtable}
\usepackage{tikz}
\usepackage{comment}
\usepackage{multirow}
\usepackage{makecell}
\usepackage{inconsolata}
\usepackage{multicol}
\usepackage{latexsym}
\usepackage{microtype}
\usepackage{multirow}
\usepackage{makecell}
\usepackage{inconsolata}
\usepackage{CJKutf8}
\usepackage{amsmath}
\usepackage{booktabs}
\usepackage{enumitem}
\usepackage{siunitx}
\usepackage{floatflt}
\usepackage{graphicx}
\usepackage{booktabs}
\usepackage{wrapfig}
\usepackage{newfloat}
\usepackage{listings}
\usepackage{authblk}
\usepackage{lipsum}
\usepackage{algorithm}
\usepackage{algorithmicx}
\usepackage{algpseudocode}
\usepackage{microtype}
\usepackage{multirow}
\usepackage{ulem}
\usepackage{booktabs} 
\usepackage{pifont}  

\input{math_commands.tex}

\usepackage{hyperref}
\usepackage{amssymb} 

\algnewcommand{\LeftComment}[1]{\Statex \(\triangleright\) #1}

\usepackage{array}
\usepackage{amsmath}
\usepackage{amssymb}
\usepackage{mathtools}
\usepackage{amsthm}

\usepackage[capitalize,noabbrev]{cleveref}
\usepackage{adjustbox} 
\usepackage{graphicx}
\theoremstyle{plain}
\newtheorem{theorem}{Theorem}[section]
\newtheorem{proposition}[theorem]{Proposition}

\theoremstyle{definition}

\theoremstyle{remark}

\colmfinalcopy

\usepackage[utf8]{inputenc}
\usepackage[T1]{fontenc}
\usepackage{caption} 
\usepackage{arydshln}
\usepackage{fontawesome5}

\usepackage{tcolorbox}
\tcbuselibrary{skins,breakable}

\tcbuselibrary{skins}

\usepackage{color}
\usepackage{xcolor}
\usepackage{soul} 

\definecolor{tred}{RGB}{251, 130, 132}
\definecolor{torange}{RGB}{247, 162, 116}
\definecolor{tyellow}{RGB}{251, 218, 140}
\definecolor{tgreen}{RGB}{127, 204, 181}
\definecolor{tblue}{RGB}{89, 177, 215}
\definecolor{insightblue}{RGB}{162, 210, 255}
\definecolor{questionred}{RGB}{255, 175, 204}

\title{LAMDAS: LLM as an Implicit Classifier for Domain-specific Data Selection}

\author{%
Jian Wu\textsuperscript{1,2}\thanks{Equal contribution. This work was done during Jian Wu's internship at Ant Group.}, Hang Yu\textsuperscript{1}$^*$, Bingchang Liu\textsuperscript{1}, Wenjie Yang\textsuperscript{1}, Peng Di\textsuperscript{1}, Jianguo Li\textsuperscript{1}, and Yue Zhang\textsuperscript{2} \\
Email: \{wujian, zhangyue\}@westlake.edu.cn, \\\{hyu.hugo, bingchang.lbc, ywj439780, dipeng.dp, lijg.zero\}@antgroup.com
\\

\vspace{10pt}
$^1$Ant Group\ \ \ $^2$Westlake University\\
\vspace{10pt}
}

\colmfinalcopy 

\begin{document}

\maketitle

\begin{abstract}
Adapting large language models (LLMs) to specific domains often faces a critical bottleneck: the scarcity of high-quality, human-curated data. While large volumes of unchecked data are readily available, indiscriminately using them for fine-tuning risks introducing noise and degrading performance. Strategic data selection is thus crucial, requiring a method that is both accurate and efficient. Existing approaches, categorized as similarity-based and direct optimization methods, struggle to simultaneously achieve these goals. In this paper, we introduce LAMDAS (\textit{L}LM \textit{A}s an i\textit{M}plicit classifier for domain-specific \textit{DA}ta \textit{S}election), a novel approach that leverages the pre-trained LLM itself as an implicit classifier, thereby bypassing explicit feature engineering and computationally intensive optimization process. LAMDAS reframes data selection as a one-class classification problem, identifying candidate data that "belongs" to the target domain defined by a small reference dataset. Extensive experimental results demonstrate that LAMDAS not only exceeds the performance of full-data training using a fraction of the data but also outperforms nine state-of-the-art (SOTA) baselines under various scenarios. Furthermore, LAMDAS achieves the most compelling balance between performance gains and computational efficiency compared to all evaluated baselines.\footnote{Code will be available upon publication.}
\end{abstract}

\section{Introduction}
The unprecedented capabilities of modern large language models (LLMs) rely on their ability to learn from vast and diverse datasets during pre-training, followed by adaptation to specialized domains or tasks via continual pre-training (CPT) and supervised fine-tuning (SFT). However, a key challenge arises when adapting LLMs to new scenarios: while high-quality, human-curated data is often scarce, large volumes of unchecked, automatically collected data, such as web-scraped content, crowd-sourced annotations, or synthetic examples, are readily accessible. Indiscriminately using this unchecked data for CPT or SFT carries the risk of introducing noise, misalignment with the target task, or even harmful patterns, leading to degraded model performance and increased hallucinations~\citep{zhang2023siren,Li2023OneSL}. Furthermore, training LLMs on unnecessarily large datasets incurs significant computational costs, requiring substantial GPU resources that are not readily available to all researchers and practitioners~\citep{he2024shed}. It is therefore imperative to strategically select the most relevant training examples from the vast pool of unchecked candidate data, particularly when working with a small but meticulously curated reference dataset that accurately represents the target domain or task.

The core challenge in domain-specific data selection lies in identifying a method that is both \textbf{accurate} – maximizing LLM performance on the reference dataset after training on the selected data – and \textbf{efficient} – capable of processing massive candidate datasets, especially for CPT. 

\begin{figure*}[htbp]
\centering
\subfloat[CPT]{%
\includegraphics[width=0.4\textwidth,height=0.34\textwidth]{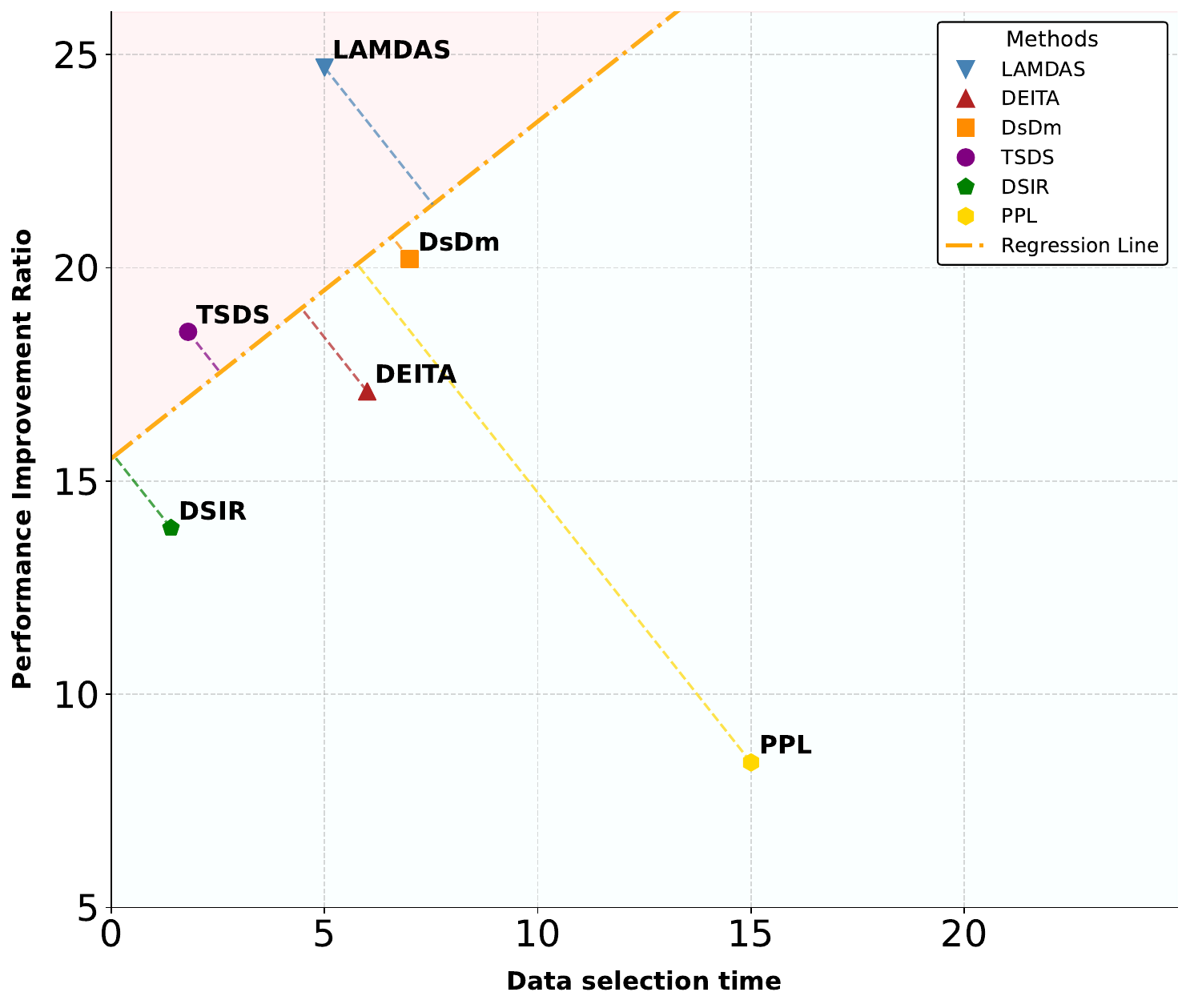}
  \label{fig:cpt}%
}
\hspace{40pt}
\subfloat[SFT]{ \includegraphics[width=0.4\textwidth,height=0.34\textwidth]{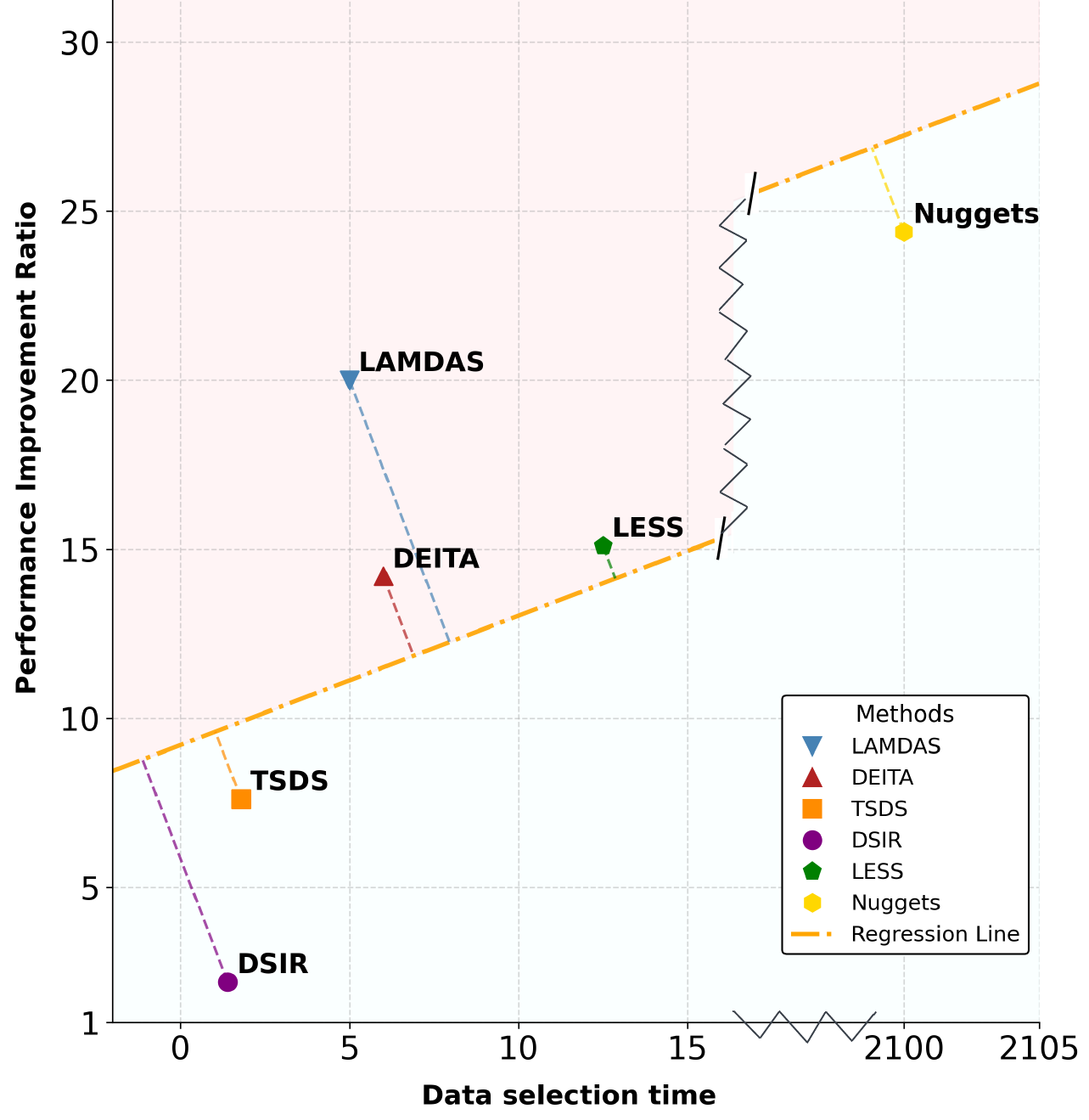}%
  \label{fig:sft}%
}
\vspace{-2pt}
\caption{\label{tradeoff} Performance gains versus data selection efficiency (selection time) for all methods in both (a) CPT and (b) SFT scenarios. A larger perpendicular distance, indicated by the dashed lines, from a method to the regression line—and above the regression line—indicates a more favorable trade-off.  The regression line is estimated using linear regression on the x and y values of all methods—suggests a more favorable trade-off. The zigzag lines mean scale breaks.}
\end{figure*}

Unfortunately, existing approaches often struggle to achieve both of these goals simultaneously. As discussed in Section~\ref{sec:related_work}, the current literature on domain-specific data selection can be divided into two primary groups: similarity-based methods and direct optimization methods. Similarity-based methods typically extract features from both candidate and reference data, selecting candidate data based on their similarity to the reference data according to specific measures. Within this category, methods using easily obtained features like lexicon or embedding-based representations~\citep{liu2024tsds, xie2023data, thrush2025improving} and computationally cheap measures like cosine similarity or correlation are fast. However, selecting data based on such superficial characteristics can be ineffective or even detrimental to model performance, as pointed out in~\citep{engstrom2024dsdm}. To tackle this issue, more recent methods utilize more informative features, such as gradients~\citep{gu2025data,yu2024mates} or model weights~\citep{xia2024less, Huang2025EvaluatingSU}, and employ advanced measures like optimal transport~\citep{liu2024tsds} and KL reduction~\citep{xie2023data}, achieving improved performance gains at the cost of efficiency. On the other hand, direct optimization methods aim to enhance the LLM's performance on the reference data by selecting suitable subsets of candidate data based on frameworks like optimal control theory~\citep{gu2025data}, data models~\citep{engstrom2024dsdm}, conditional loss~\citep{brandfonbrener2024color}, or Shapley values~\citep{he2024shed}. However, the high computational burden of these methods typically requires them to estimate selection scores for only a representative subset of the candidate data. A smaller or simplified model is then trained to extrapolate these scores to the entire candidate dataset, adding a layer of approximation that can degrade performance. In essence, these techniques often remain computationally expensive.

In this paper, we introduce \textbf{LAMDAS}, innovatively using \textbf{L}LM \textbf{A}s an i\textbf{M}plicit classifer for domain-specific \textbf{DA}ta \textbf{S}election. The proposed method is applicable to both data selection for CPT and SFT. Our approach effectively addresses the trade-off between performance and efficiency. The key insight is to reframe data selection as a one-class classification (OCC) problem: given a small reference dataset representing the target domain (positive samples), how can we identify candidate data that "belongs" to the same class or domain? Unlike prior works, LAMDAS leverages the pre-trained LLM itself as an \textbf{implicit} classifier, bypassing the need for explicit feature engineering or computationally intensive retraining. We first demonstrate that a prefix tuned on the reference dataset acts as a concise representation of the target domain – a "domain prefix". LAMDAS scores candidate examples by comparing their likelihood under the LLM \textbf{with} and \textbf{without} this domain prefix. Candidates exhibiting significantly higher likelihoods when conditioned on the prefix are prioritized, indicating a stronger alignment with the reference distribution. Since prefix tuning only modifies the prefix tokens while leaving the LLM's weights untouched, LAMDAS effectively preserves the LLM's general knowledge and prevents overfitting to the limited reference data. Moreover, given the relative simplicity of binary classification, a smaller LLM can often achieve good classification performance (as shown in our experiments), thereby ensuring the efficiency of our approach. 

Extensive experimental results demonstrate that LAMDAS not only surpasses full data training while utilizing only a small fraction of the data but also outperforms \textbf{seven} SOTA baselines for CPT data selection and \textbf{nine} baselines for SFT data selection. Moreover, as depicted in Figure \ref{tradeoff}, LAMDAS achieves \textbf{the most favorable trade-off} between performance gains and efficiency among all baselines evaluated. In summary, the key contributions of this work are:
\begin{itemize}
    \item We innovatively formulate the domain-specific data selection problem as an OCC problem.
    \item We propose leveraging small LLMs as implicit classifiers to address the OCC problem, striking a compelling balance between efficacy and efficiency.
    \item We demonstrate through extensive experimentation on both CPT and SFT data selection tasks that LAMDAS outperforms other SOTA baselines.
\end{itemize}

\section{Related Works}\label{sec:related_work}

As mentioned in the introduction, domain-specific data selection can be broadly categorized into two groups: similarity-based methods and direct optimization methods.

\subsection{Similarity-based Methods}
Similarity-based methods typically involve three primary steps: extracting features from both candidate and reference datasets, measuring the similarity between these two sets, and selecting candidate samples that closely match the reference data. Early approaches often relied on simple features and measures, including lexicon-based features \citep{10.1016/j.knosys.2024.111740}, embedding representations \citep{liu2024tsds}, and human-crafted rules \citep{wettig2024qurating, Sachdeva2024HowTT}. The similarity measures used in these methods include binary grammar discriminators \citep{Sachdeva2024HowTT, Touvron2023LLaMAOA}, rules defined by LLMs \citep{Li2024RulebasedDS}, cosine similarity \citep{rubin-etal-2022-learning}, and perplexity \citep{marion2023less, muennighoff2023scaling}. Although these methods are efficient and impose low computational overhead, they tend to capture only superficial features of candidate data and fail to account for the intricate relationships between candidate and reference datasets. As pointed out in~\citep{engstrom2024dsdm}, reliance on simplistic features and measures can adversely affect model performance. To address these limitations, researchers have investigated ways to enhance the complexity of either features or similarity measures. For example, some methods incorporate gradients \citep{Evans2024DataVW, everaert2024gio, Zhao2025BeyondSA}, as the gradient of selected data is expected to align with that of the reference data, thereby ensuring that loss on the reference data decreases upon training with the selected samples. An alternative approach, Grad-Mimic \citep{Huang2025EvaluatingSU}, selects data that aligns its gradient with a direction that points toward the reference model in weight space. However, computing gradients can be computationally expensive, which poses scalability challenges for these methods. To alleviate this issue, SkMM \citep{Dong2024SketchyMM} employs gradient sketching, while LESS \citep{xia2024less} uses low-rank adaptation (LoRA) gradients as a low-dimensional approximation of the original gradient. MATES~\citep{yu2024mates} instead trains a small model to approximate the mapping between the candidate data and its influence score. Apart from using more advanced features, more effective measures have also been explored, such as optimal transport \citep{Kaplan2020ScalingLF, liu2024tsds}
and KL reduction \citep{xie2023data}. These techniques seek to align the distribution between the selected and the reference in the feature space, improving performance at the expense of efficiency.

\subsection{Direct Optimization Methods}
Direct optimization methods allow for direct interaction between the candidate and reference data, with the goal of maximizing the performance on the reference dataset when a model is trained using the selected data. SHED \citep{he2024shed} uses the Shapley value framework to attribute loss reduction to the candidate data, requiring training the model on different subsets of the candidate data and then computing the attribution score. Due to the extreme computational cost, SHED must first cluster the candidate data and then compute the attribution score for each cluster based on the score of the data point at the cluster center. Alternatively, PDS \citep{gu2025data} formulates data selection as an optimal control problem, using Pontryagin’s Maximum Principle (PMP) to derive necessary conditions for optimal data selection. However, solving the PMP equations is also computationally expensive, so PDS only solves the equation for representative samples within the candidate set and then approximates the solution using a small model. The resulting small model is then used for data selection, which may introduce approximation errors. Similarly, DsDm \citep{engstrom2024dsdm} frames data selection as an optimization problem using "data models" to approximate how the learning algorithm uses training data to make predictions on the target tasks. In addition, Nuggets \citep{Li2023OneSL} selects candidates that increase the generation probability of the reference data by comparing the likelihood of the reference data with and without the candidate as a one-shot example in the prompt, requiring a comparison of every candidate-reference data pair. In summary, these methods can improve performance by enabling direct interaction between candidate and reference data, but they often suffer from high computational cost.

In contrast to the above two categories, our proposed LAMDAS method aims to increase efficiency while still allowing for meaningful interaction between the candidate and reference data to ensure strong performance. We will elaborate on the relationship between LAMDAS and existing methods in Section~\ref{ssec:rel2others} after we formally introduce our methodology.

\section{Methodology}
Let $\gD_{ref} = \{\vx_i\}_{i=1}^N$ denote the reference dataset and $\gD_{cand} = \{\vy_j\}_{j=1}^M$ denote the candidate dataset. We frame the data selection problem as a classification task, where our objective is to train a classifier to determine whether each sample in the candidate dataset should be retained. However, a significant challenge arises in that we possess only positive samples from the reference dataset during training. This scenario defines the problem as a one-class classification (OCC) problem, where we must identify relevant samples without negative examples for guidance.

\subsection{LLM as an Implicit Classifier}
In the natural language processing domain, classification tasks typically utilize one of two frameworks: BERT-based encoders or GPT-based decoders. BERT-based methods typically fine-tune the model for classification, but they require negative samples, making them unsuitable for the OCC setting. Conversely, GPT-based methods, as exemplified by Ask-LLM~\citep{Sachdeva2024HowTT}, could potentially select data by prompting the text decoder with a natural language description of the class and asking whether the specific sample belongs to this class. However, in our scenario, we lack a clear natural language description of the target domain; we only have the representative examples in the reference data. Summarizing the reference data into a concise and accurate textual description can be difficult because text summarization inevitably loses some information through the bottleneck of human-interpretable natural language. Furthermore, some domains, such as emerging slang dialects, exhibit patterns that are difficult to capture in formal language but manifest clearly in data.
To circumvent the above limitations, we propose to harness the LLM itself as an implicit classifier.

\subsubsection{Prefix Tuning for Domain Representation Learning} We represent the target domain $\mC$ (defined by $\gD_{ref}$) via a learnable domain prefix $\mC$, which steers the LLM’s generation toward the reference distribution. Specifically, we learn a soft prefix $\mC$ that maximizes the likelihood of the reference data:
\begin{align} 
\mC &= \argmax \log p(\vx_{1:N}|\mC)\notag \\ 
&= \argmax \sum_{i=1}^N \log p(\vx_i|\mC). \label{eq:C_ml}
\end{align}
Here we opt for prefix tuning due to its greater flexibility compared to prompt tuning. By employing a learned soft prefix, we avoid the need to manually craft domain descriptions, a requirement that hinders the use of Ask-LLM. Furthermore, the above estimation inherently maximizes the mutual information between the reference data and the learned domain representation $\mC$. Given the typically small size of the reference dataset, the prefix can effectively retain essential information about the domain defined by that dataset.

\subsubsection{Likelihood Ratio for Data Selection}
Once we have learned the domain representation $\mC$, we can estimate the probability that a candidate sample belongs to the same class or domain, leveraging Bayes' theorem:
\begin{align} \label{eq:bayes}
p(\mC|\vy_j) = \frac{p(\vy_j|\mC) p(\mC)}{p(\vy_j)}.
\end{align}
Here, we can disregard the prior $p(\mC)$ because it remains constant across all samples $\vy_j$ in the candidate set, and does not assist in differentiating between selected and unselected data. Consequently, we define the sample selection score using the likelihood ratio of $\vy_j$ from the LLM with and without the prefix or class representation $\mC$:
\begin{align}
s(\vy_j) = \frac{p(\vy_j|\mC)}{p(\vy_j)}.
\end{align}
We select $\vy_j$ when the score exceeds the threshold $\tau = 1$, thus prioritizing samples that align more closely with the representation of the reference distribution. Indeed,
\begin{proposition} (Likelihood Ratio as Domain Discriminant): 
The score $s(\vy_j)$ is equivalent to the likelihood ratio test statistic for distinguishing $p(\vy_j|\mC)$ from $p(\vy_j)$. This likelihood ratio test is optimal for binary hypothesis testing with fixed Type I or Type II error rates.
\end{proposition}

\begin{proof}
This proposition is supported by the Neyman-Pearson lemma~\citep{casella2024statistical}.
\end{proof}

The formulation presented leverages the LLM as an implicit classifier without modifying the LLM's weights. This strategy preserves the rich pre-trained knowledge embedded in the LLM, ensuring that our method is generalizable to unseen samples within the candidate set. Additionally, our approach allows for seamless application across different domains by simply altering the domain prefix. Given that classification tasks are relatively easy, even small LLMs can perform well in practice, further enhancing efficiency.

\subsection{Relation to Other Methods}
\label{ssec:rel2others}
\subsubsection{Balancing Efficacy and Efficiency}
As previously discussed, strategies for domain-specific data selection can be categorized into two main groups: similarity-based and direct optimization methods. In fact, similarity-based methods can be viewed as bi-encoder or dual-encoder approaches in information retrieval~\citep{muennighoff2022sgpt}, 
where features for both candidate and reference data are independently extracted via encoders and then compared using similarity measures. This constrained interaction limits their ability to fully capture relationships between the two data sets, often resulting in suboptimal performance. In contrast, direct optimization methods facilitate direct interaction between candidate and reference data, akin to cross-encoder approaches~\citep{liao-etal-2024-d2llm}, thus improving performance. For example, NUGGETS allows for interaction between every pair $(\vx_i, \vy_j)$ by computing the likelihood ratio of $p(\vy_j | \vx_i) / p(\vx_i)$. However, such methods often significantly increase computational costs.

Our proposed method seeks to strike a balance between these two approaches. Initially, we condense the reference data into a concise domain representation or prefix, which then allows for direct interaction between all candidate data and this domain representation. The maximum likelihood estimation of the domain representation (see Eq.~\eqref{eq:C_ml}) theoretically maximizes the mutual information between the domain prefix and the reference data. Given that reference datasets typically have limited sample sizes, it is straightforward for the prefix to retain critical information. The direct interaction between the prefix and each candidate sample ensures that performance exceeds that of similarity-based methods, as cross-encoder methods typically deliver superior outcomes compared to bi-encoder techniques. Moreover, the concise domain prefix eliminates the need for exhaustive interactions between every candidate and reference data pair, thus enhancing efficiency in comparison to direct optimization methods.

\subsubsection{Connection to Classifier-Free Guidance}
Our method draws inspiration from classifier-free guidance employed in diffusion models~\citep{Yang2022DiffusionMA}, which uses a generative diffusion model as an implicit classifier to steer the text-to-image generation process, ensuring that generated images can be classified into the domains defined by provided text. In our case, we utilize the LLM as an implicit classifier for data selection, ensuring that the selected subset from the candidate dataset accurately reflects the domain defined by the reference data.
\section{Experiments}

We conduct comprehensive experiments to validate LAMDAS across continued pre-training (CPT) and supervised fine-tuning (SFT) scenarios. Our experiments aim to answer the following questions: 
\begin{enumerate}
    \item \textbf{Performance Gains}: Does LAMDAS outperform existing methods in downstream tasks when selecting the same amount of data?
    \item \textbf{Efficiency}: Does LAMDAS achieve lower time complexity for data selection compared with existing methods?
    \item \textbf{Sensitivity}: Is LAMDAS's performance sensitive to the selection threshold $\tau$, the length of the prefix, the size of the classifier, and the base LLM for the classifier?
\end{enumerate}

\begin{table*}[t]
\centering

\footnotesize
\resizebox{1.0\textwidth}{!}{
\begin{tabular}{ccccccccccc}
\toprule
Methods  & Data Size  & Models & HE & HE+  & MBPP & MBPP+  & LCB &CRUXEval & AVG\\
\midrule
& & Qwen2.5-0.5B & { {\textbf{23.5}}} & { {\textbf{19.4}}}  & {{\textbf{43.9}}} & { {\textbf{37.8}}}&{ {\textbf{5.8}}}  & { {\textbf{4.9}}} & {{\textbf{23.5}}}\textit{\scriptsize{+52.6\%}}\\

& & Qwen2.5-1.5B & { {\textbf{26.7}}} &{ {\textbf{22.3}}} & { {\textbf{50.5}}} & { {\textbf{ 44.3}}}& { {\textbf{13.7}}} & { {\textbf{9.7}}} & { {\textbf{27.9}}}\textit{\scriptsize{+46.1\%}}\\
\multirow{-3}{*}{LAMDAS(ours)}& \multirow{-3}{*}{15B}  & Qwen2.5-7B   & { {\textbf{ 35.2}}}& { {\textbf{29.2}}}& { {\textbf{58.5}}}&{ {\textbf{50.6}}} &{ {\textbf{25.7}}} & { {\textbf{15.8}}}  &{ {\textbf{35.8}}}\textit{\scriptsize{+24.7\%}}\\

\midrule
& & Qwen2.5-0.5B & 22.1 & 17.6 & 42.5 & \uwave{35.9} & 4.1 & 3.9 & 21.5\textit{\scriptsize{+39.6\%}}\\
& & Qwen2.5-1.5B & 25.8 & 21.2 & 48.7 & 43.1 &9.9 & 8.3 & 26.4\textit{\scriptsize{+38.2\%}}\\
\multirow{-3}{*}{DEITA \citep{liu2024what}} & \multirow{-3}{*}{15B}  & Qwen2.5-7B & 34.1 & 27.8 & \uwave{56.1} & \underline{49.2} & 21.3 & 14.9 & 33.6\textit{\scriptsize{+17.1\%}}\\

\midrule
& & Qwen2.5-0.5B & \uwave{23.2} & \uwave{18.6} & 39.3 & 35.4 & \uwave{4.9} & \uwave{4.2} &\uwave{20.9}\textit{\scriptsize{+35.7\%}}\\
& & Qwen2.5-1.5B & \uwave{25.8} & \uwave{21.5} & \uwave{47.6} & \uwave{43.8} & \uwave{12.5} & \uwave{8.7} &\uwave{26.7}\textit{\scriptsize{+39.8\%}}\\
\multirow{-3}{*}{DsDm\citep{engstrom2024dsdm}} & \multirow{-3}{*}{15B}  & Qwen2.5-7B & \uwave{34.7} & \uwave{28.2} & \underline{55.8} & \uwave{49.4} & \uwave{23.5} & \uwave{15.2} & \uwave{34.5}\textit{\scriptsize{+20.2\%}}\\

\midrule
& & Qwen2.5-0.5B & \underline{22.9} & \underline{19.1} & \uwave{41.3} & \uwave{36.4} & \underline{5.3}& \underline{4.6} & \underline{21.6}\textit{\scriptsize{+40.3\%}}\\
& & Qwen2.5-1.5B & \uwave{26.2} & \underline{21.9} & \underline{49.8} & \underline{44.1} & \underline{13.4} & \underline{8.8} & \underline{27.4}\textit{\scriptsize{+43.5\%}}\\
\multirow{-3}{*}{TSDS \citep{liu2024tsds}}  & \multirow{-3}{*}{15B}  & Qwen2.5-7B & \underline{34.2} & \underline{27.2} & 54.7 & 48.6 & \underline{24.2}& \underline{15.3} & \underline{34.0}\textit{\scriptsize{+18.5\%}}\\

\midrule
 & & Qwen2.5-0.5B & 21.5 & 17.3 & \underline{41.2} & 34.8 & 3.89 & 3.5 & 20.4\textit{\scriptsize{+32.5\%}}\\
& & Qwen2.5-1.5B & 24.2 & 20.1 & 46.3 & 41.8 & 9.1 &7.1 & 24.8\textit{\scriptsize{+29.8\%}}\\
\multirow{-3}{*}{DSIR \citep{xie2023data}} & \multirow{-3}{*}{15B}  & Qwen2.5-7B & 33.5 & 26.3 & 53.1 & 47.8 & 20.9 & 12.5 & 32.7\textit{\scriptsize{+13.9\%}}\\

\midrule
 & & Qwen2.5-0.5B & 20.7 & 15.4 & 39.2 & 33.7 & 3.6 & 3.4 &19.3\textit{\scriptsize{+25.3\%}}\\
& & Qwen2.5-1.5B & 22.1 & 19.4 & 45.7 & 41.4 & 10.2 & 6.6 & 24.2\textit{\scriptsize{+26.7\%}}\\
\multirow{-3}{*}{PPL} & \multirow{-3}{*}{15B}  & Qwen2.5-7B & 30.6 & 27.5 & 51.7 & 46.3 & 19.2 & 11.3 & 31.1\textit{\scriptsize{+8.4\%}}\\

\midrule
& & Qwen2.5-0.5B & 19.7 & 14.8 & 38.6 & 32.4 & 3.41 & 3.3 & 18.7\textit{\scriptsize{+21.4\%}}\\
& & Qwen2.5-1.5B & 21.7 & 17.9 & 43.7 & 38.9 & 8.6 & 6.5 & 22.9\textit{\scriptsize{+19.9\%}}\\
\multirow{-3}{*}{Random} & \multirow{-3}{*}{15B}  & Qwen2.5-7B & 29.5 & 26.1 & 49.2 & 44.7 & 18.3 & 9.7 & 29.6\textit{\scriptsize{+3.1\%}}\\

\midrule
& & Qwen2.5-0.5B & 12.8 & 9.3 & 35.1 & 28.6 & 3.3 & 3.0 & 15.4\\
& & Qwen2.5-1.5B & 17.8 & 13.2 & 38.1 & 30.3 & 7.9 & 7.2 & 19.1\\
\multirow{-3}{*}{Full} & \multirow{-3}{*}{100B} & Qwen2.5-7B & 27.2 & 22.1 & 52.9 & 43.4 & 16.4 & 9.9 & 28.7 \\    

\midrule
& & Qwen2.5-0.5B & 16.5 & 12.3 & 37.1 & 30.6 & 5.2 & 6.0 & 18.0\\
& & Qwen2.5-1.5B & 21.2 & 17.3 & 41.8 & 37.9 & 8.1 & 7.7 & 21.1\\
\multirow{-3}{*}{Base} & \multirow{-3}{*}{0B} & Qwen2.5-7B & 28.9 & 22.1 & 45.9 & 40.4 & 14.4 & 8.9 & 26.8 \\   
\bottomrule
\end{tabular}}
\caption{\label{CPT_main_results}Pass@1 performance on HumanEval (HE), MBPP, LCB, and CRUXEval. All models are continually pretrained on selected data and evaluated with a zero-shot prompt strategy. We also compute the average of models' performance on all benchmarks and calculate the proportion of improvement compared with the "Full" data. The best results are highlighted in {\textbf{bold}}, the second-best results are \uwave{underlined}, and the third are \underline{wavy underlined}.}
\end{table*}

To answer our research questions, we design a comprehensive experimental setup for both coding and mathematical reasoning across CPT and SFT scenarios. We use high-quality reference data to select from large candidate pools and evaluate performance on a suite of relevant benchmarks chosen to test generalization. LAMDAS is benchmarked against several groups of baselines: seven methods applicable to both CPT and SFT, including similarity-based methods (DSIR, TSDS), direct optimization methods (DsDm), and perplexity-based (PPL) approaches; two additional direct optimization methods for the smaller SFT datasets (Nuggets, LESS) that are too computationally intensive for CPT; and standard controls (Random, Full). For brevity, detailed specifications of the datasets, baseline implementations, and hyperparameters are provided in Appendix \ref{sec:experiment_setup}.

\begin{table*}[htb]
\centering
\footnotesize
\resizebox{1.0\textwidth}{!}{
\begin{tabular}{cccccccccc}
\toprule
Methods & Data Size & Models & HE & HE+ & MBPP & MBPP+ & LCB & CRUXEval & AVG \\
\midrule
\multirow{4}{*}{LAMDAS(ours)} & \multirow{4}{*}{750K} & Qwen2.5-0.5B-Instruct & {\textbf{25.3}} & {\textbf{23.1}} & {\textbf{48.2}} & {\textbf{42.1}} & {\textbf{8.7}} & {\textbf{15.1}} & {\textbf{27.1}}\textit{\scriptsize{+36.9\%}} \\
& & Qwen2.5-1.5B-Instruct & { {\textbf{32.8}}}& { {\textbf{30.1}}}& { {\textbf{55.6}}}&{ {\textbf{47.1}}}& { {\textbf{17.2}}}& { {\textbf{25.5}}}& { {\textbf{34.3}}}\textit{\scriptsize{+30.3\%}}\\
 & & Qwen2.5-7B-Instruct & {\textbf{42.8}} & {\textbf{37.6}} & {\textbf{62.1}} & {\textbf{55.3}} & {\textbf{32.5}} & {\textbf{37.2}} & {\textbf{44.6}}\textit{\scriptsize{+19.6\%}} \\
 & & Qwen2.5-32B-Instruct & {\textbf{89.0}} & {\textbf{84.1}} & {\textbf{75.3}} & {\textbf{69.4}} & {\textbf{44.1}} & {\textbf{52.3}} & {\textbf{69.0}}\textit{\scriptsize{+20.0\%}} \\
\midrule
\multirow{4}{*}{DSIR\citep{xie2023data}} & \multirow{4}{*}{750K} & Qwen2.5-0.5B-Instruct & 21.2 & 18.5 & 36.7 & 31.5 & 5.5 & 11.2 & 20.8\textit{\scriptsize{+5.1\%}} \\
& & Qwen2.5-1.5B-Instruct & 27.8 & 23.1 & 43.2 & 35.4 & 13.6 & 17.6 & 26.8\textit{\scriptsize{+1.5\%}}\\
 & & Qwen2.5-7B-Instruct & 38.9 & 32.1 & 52.3 & 47.6 & 24.4 & 30.3 & 37.6\textit{\scriptsize{+0.8\%}} \\
 & & Qwen2.5-32B-Instruct & 80.3 & 72.7 & 61.9 & 56.3 & 40.5 & 47.3 & 59.8\textit{\scriptsize{+2.2\%}} \\
\midrule
\multirow{4}{*}{TSDS \citep{liu2024tsds}} & \multirow{4}{*}{750K} & Qwen2.5-0.5B-Instruct & 23.4 & 21.5 & 40.7 & 34.5 & 6.8 & 11.9 & 23.1\textit{\scriptsize{+16.7\%}} \\
& & Qwen2.5-1.5B-Instruct & \underline{30.8} & 25.2 & 49.3 & \underline{43.2} & 14.8 & 18.3 & 30.3\textit{\scriptsize{+14.8\%}}\\
 & & Qwen2.5-7B-Instruct & \underline{40.1} & 34.2 & 55.6 & 50.2 & 27.3 & 32.1 & 39.9\textit{\scriptsize{+7.0\%}} \\
 & & Qwen2.5-32B-Instruct & 82.5 & 75.9 & 63.4 & 59.2 & 41.2 & 48.9 & 61.9\textit{\scriptsize{+7.6\%}} \\
\midrule
\multirow{4}{*}{DEITA \citep{liu2024what}} & \multirow{4}{*}{750K} & Qwen2.5-0.5B-Instruct & 23.7 & 21.8 & 46.1 & 40.5 & 7.3 & 13.5 & 25.5\textit{\scriptsize{+28.8\%}} \\
& & Qwen2.5-1.5B-Instruct & 29.3 & 25.2 & 51.4 & 42.9 & \underline{15.7} & 21.1 & 30.9\textit{\scriptsize{+17.0\%}}\\
 & & Qwen2.5-7B-Instruct & 39.8 & 33.2 & 59.7 & 51.8 & 29.3 & 34.7 & 41.4\textit{\scriptsize{+11.0\%}} \\
 & & Qwen2.5-32B-Instruct & 86.1 & 80.2 & 72.9 & 65.9 & 41.7 & \underline{48.1} & 65.7\textit{\scriptsize{+14.2\%}} \\
\midrule
\multirow{4}{*}{DsDm\citep{engstrom2024dsdm}} & \multirow{4}{*}{750K} & Qwen2.5-0.5B-Instruct & \underline{24.2} & 21.7 & \underline{47.5} & \underline{41.8} & 7.6 & \underline{14.3} & \underline{26.2}\textit{\scriptsize{+32.3\%}} \\

& & Qwen2.5-1.5B-Instruct & \uwave{31.4} & \uwave{29.5} & 51.2 & 42.3 & 15.5 & 24.1 & \uwave{32.3}\textit{\scriptsize{+22.3\%}}\\
 & & Qwen2.5-7B-Instruct & 39.8 & \underline{34.7} & \underline{60.3} & 50.8 & 29.6 & \underline{36.2} & \underline{41.9}\textit{\scriptsize{+12.3\%}} \\
 & & Qwen2.5-32B-Instruct & \uwave{87.2} & \uwave{81.1} & \uwave{73.5} & \uwave{67.1} & \uwave{42.9} & 47.6 & \uwave{66.6}\textit{\scriptsize{+15.8\%}} \\
\midrule
\multirow{4}{*}{Nuggets\citep{Li2023OneSL}} & \multirow{4}{*}{750K} & Qwen2.5-0.5B-Instruct & 24.1 & \uwave{22.7} & 43.4 & 39.4 & \underline{7.7} & \underline{14.8} & 25.4\textit{\scriptsize{+28.3\%}} \\
& & Qwen2.5-1.5B-Instruct & 29.8 & 24.9 & \underline{51.7} & 42.9 & 15.4 & \uwave{23.9} & 31.4\textit{\scriptsize{+18.9\%}}\\
 & & Qwen2.5-7B-Instruct & 38.7 & 33.6 & 57.5 & \underline{52.1} & \underline{29.8} & \uwave{37.0} & 41.5\textit{\scriptsize{+11.3\%}} \\
 & & Qwen2.5-32B-Instruct & \underline{86.4} & 80.0 & 72.9 & 65.6 & 42.1 & 47.5 & 65.8\textit{\scriptsize{+14.4\%}} \\
\midrule
\multirow{4}{*}{LESS \citep{xia2024less}} & \multirow{4}{*}{750K} & Qwen2.5-0.5B-Instruct & \uwave{24.5} & \underline{22.3} & \uwave{48.2} & \uwave{43.1} & \uwave{8.4} & \uwave{13.8} & \uwave{26.7}\textit{\scriptsize{+34.8\%}} \\
& & Qwen2.5-1.5B-Instruct & 30.2 & \underline{26.1} & \uwave{53.2} & \uwave{44.8} & \uwave{15.8} &\underline{22.8} & \underline{32.2}\textit{\scriptsize{+3.1\%}}\\
 & & Qwen2.5-7B-Instruct & \uwave{40.2} & \uwave{35.1} & \uwave{61.1} & \uwave{54.2} & \uwave{30.9} & 35.7 & \uwave{42.9}\textit{\scriptsize{+15.0\%}} \\
 & & Qwen2.5-32B-Instruct & \underline{86.7} & \underline{80.5} & \underline{73.1} & \underline{66.2} & \underline{42.3} & \uwave{48.4} & \underline{66.2}\textit{\scriptsize{+15.1\%}} \\
\midrule
\multirow{4}{*}{Random} & \multirow{4}{*}{750K} & Qwen2.5-0.5B-Instruct & 18.6 & 15.8 & 32.5 & 28.7 & 5.4 & 8.1 & 18.2\textit{\scriptsize{-0.8\%}} \\
& & Qwen2.5-1.5B-Instruct & 24.5 & 21.2 & 37.8 & 33.1& 14.0 & 15.3 & 24.3\textit{\scriptsize{-0.8\%}}\\
 & & Qwen2.5-7B-Instruct & 35.6 & 30.3 & 49.2 & 43.2 & 23.5 & 27.9 & 35.0\textit{\scriptsize{-0.6\%}} \\
 & & Qwen2.5-32B-Instruct & 76.2 & 68.9 & 56.7 & 51.2 & 35.1 & 39.8 & 54.7\textit{\scriptsize{-4.9\%}} \\
\midrule
\multirow{4}{*}{Full} & \multirow{4}{*}{1800K} & Qwen2.5-0.5B-Instruct & 20.1 & 16.5 & 34.1 & 31.6 & 4.4 & 12.3 & 19.8 \\
& & Qwen2.5-1.5B-Instruct & 27.4 & 22.5 & 40.2 & 35.2& 12.7 & 20.6 & 26.4\\
 & & Qwen2.5-7B-Instruct & 37.6 & 31.8 & 51.2 & 45.6 & 25.2 & 32.2 & 37.3 \\
 & & Qwen2.5-32B-Instruct & 77.5 & 70.4 & 60.1 & 54.3 & 38.5 & 44.2 & 57.5 \\
\midrule
\multirow{4}{*}{Base} & \multirow{4}{*}{0K} & Qwen2.5-0.5B-Instruct & 15.2 & 10.5 & 26.7 & 21.9 & 3.6 & 6.7 & 14.1 \\
& & Qwen2.5-1.5B-Instruct & 22.3 & 17.6 & 35.1 & 31.4 & 8.9 & 15.4 & 21.8 \\
 & & Qwen2.5-7B-Instruct & 31.1 & 25.8 & 42.1 & 35.7 & 19.5 & 24.8 & 29.8 \\
 & & Qwen2.5-32B-Instruct & 72.1 & 65.5 & 53.4 & 47.9 & 30.4 & 37.1 & 51.1 \\
\bottomrule
\end{tabular}}
\caption{\label{SFT_code_results}Performance of Qwen2.5 serious models trained on various data selection methods on HumanEval (HE), MBPP, LiveCodeBench, and CRUXEval (Zero-shot, Pass@1). We also compute the average of models' performance on all benchmarks and calculate the proportion of improvement compared with the "Full" data. The best results are highlighted in {\textbf{bold}}, the second-best results are \uwave{ underlined}, and the third are \underline{highlighted}.}
\vspace{-0.45cm}
\end{table*}

The next two subsections present the evaluation results comparing LAMDAS with several baseline methods. Our experiments focus on assessing the capabilities of the Qwen2.5 series models when trained using datasets selected by LAMDAS and the baseline methods. We examine performance on coding tasks for CPT, and on both coding and mathematical reasoning tasks for SFT.

\subsection{Data Selection for CPT}
The results demonstrate that CPT with data selected by LAMDAS consistently achieves the best performance across all metrics, regardless of the underlying model. Indeed, when using the selected data to train the Qwen2.5-7B model, LAMDAS outperforms the same model trained on the "Full" data by an average of 24.7\%, while simultaneously reducing the dataset size by a significant 75\%. Similarly, Qwen2.5-0.5B and Qwen-1.5B resulting from LAMDAS also outperform the same models trained on the "Full" data by an average of 52.6\% and 46.1\%. As visualized in Figure~\ref{distribution}, LAMDAS effectively concentrates on the region that closely matches the reference distribution, which explains its superior performance in downstream tasks.

The second-best method, DsDm, focuses on direct optimization to minimize loss on the reference set by selecting candidate data based on influence scores. However, its high computational cost necessitates the use of a simplified proxy estimator to approximate these scores, trading approximation error for efficiency~\citep{10.5555/3692070.3692568}. 

The third-best method, TSDS, employs text embeddings as features and uses optimal transport as a measure of similarity. TSDS's better performance compared to DSIR illustrates the importance of careful feature selection in similarity-based approaches, as DSIR relies solely on superficial features like n-grams. Despite this, both similarity-based methods still fall short of LAMDAS's performance. This can be attributed to the fact that feature extraction and similarity computation, even with advanced features and measures, do not offer the same flexibility as LAMDAS. LAMDAS facilitates direct interaction between candidate and reference domains through a cross-attention mechanism in the LLM, creating connections between the domain prefix and candidate samples. 

In contrast, DEITA selects data based on quality and complexity scores, prioritizing factors other than domain relevance. This explains its underperformance on domain-specific tasks. While PPL may implicitly consider domain relevance, it primarily selects data where next tokens are easily predictable given previous tokens, driven by its objective of minimizing next-token prediction loss (see Table~\ref{case_study} in the appendix).
Interestingly, the "Random" baseline outperforms the "Full" baseline in our CPT experiments. We attribute this to the nature of the CPT data, which consists of pure code data from GitHub and The Stack v1. Continuously pre-training with this entire dataset can lead to catastrophic forgetting of the model's NLP capabilities (as shown in Table \ref{CPT_main_results}). The evaluation benchmarks, however, require instruction following for code completion and text-to-code tasks. The "Random" baseline mitigated this issue by only pre-training the model with 15B tokens, thus partially preserving the model's NLP abilities and leading to better results. Furthermore, other data selection methods further improved the results of "Random", showcasing their effectiveness.

\subsection{Data Selection for SFT}

\subsubsection{Code}  Once again, as shown in tables \ref{SFT_code_results}, the Qwen2.5-32B-Instruct model trained on data selected by LAMDAS demonstrates superior performance, exceeding the "Full" method by 20.0\% and outpacing the second-best method, LESS, by an average of 15.1\%. LESS and DsDm perform comparably, achieving the second and third best results, as both methods aim to select training data that minimizes the model's loss on the reference data. However, both face challenges related to high time complexity due to the expensive computation of gradients, as discussed in Section~\ref{sec:efficiency}. Additionally, LESS suffers from the limitations inherent in similarity-based methods, while DsDm improves efficiency by using small proxy models for gradient computation, though this can introduce approximation error. Consequently, both methods perform worse than LAMDAS. Similarly, although Nuggets is a direct optimization method, it still underperforms relative to LAMDAS. This shortfall may stem from Nuggets' reliance on the original Llama2-7B model to assess the relevance of candidate samples to the reference data, which may not fully capture the nuances of the specific domain. Conversely, LAMDAS utilizes prefix tuning for domain adaptation before evaluating candidate samples for relevance, enhancing its effectiveness. Different from the CPT experiment, "Full" performs better than "Random" here. This may be because all candidate data are related to coding tasks with natural language instructions, eliminating the risk of catastrophic forgetting of instruction-following capabilities when training with purely code data.

\subsubsection{Math} The performance trends for the math reasoning task are similar to those in coding tasks; detailed discussion is provided in Appendix \ref{app:math}.

\subsection{Efficiency}
\label{sec:efficiency}

\begin{table*}[t]
\small
\centering
\begin{tabular}{lllll}
\toprule
Models          & LAMDAS                 & DsDm          & TSDS            & DSIR       \\ \midrule
Time Complexity & $2*C_{forward}$ & $2*C_{forward} + C_{grad}$ & $\epsilon$      & $\epsilon$ \\
Run Time s/100      & 5s             &  7s   & 1.8s & 1.4s    \\ 
Reference Model Size &  0.5B & 0.5B & 161M & - \\
\midrule
Models          & PPL             & DEITA           & Nuggets         & LESS                      \\
\midrule
Time Complexity & $C_{forward}$   & $2*C_{forward}$ & $|D_{ref}|*C_{forward}$ & $4*(C_{forward}+C_{grad})$\\
Run Time s/100  & 15s & 6s & 2,100s  & 12.5s   \\
Reference Model Size & 7B & 7B & 7B & 0.5B\\
\bottomrule
\end{tabular}
\caption{\label{time_complexity} Time Complexity and Run Time of all methods. The run time is calculated based on the selection time for the same 100 samples. The time complexity is approximated as the number of model forward ($C_{forward}$). $\epsilon$ means the constant time complexity. The $C_{grad}$ means the time cost
of computing gradient. We also illustrate the reference model size. Since DSIR is N-gram based method, which model size is shown as ``-'' }
\end{table*}
As shown in Table~\ref{time_complexity}, LAMDAS is only slower than the similarity-based methods TSDS and DSIR that are based on simple features such as text embeddings or n-grams, but is faster than all the remaining baselines. Indeed, with a time complexity of 2× model forward passes ($C_{forward}$) and an execution speed of 5 seconds per 100 samples, LAMDAS with Qwen2.5-Coder-0.5B model as the reference model achieves a 1.4× speedup over direct optimization method like DsDm (0.5B model, 7s/100) and 2.5× over the similarity-based method LESS (0.5B model, 12.5s/100). The table also illustrates that the actual computational time aligns with the theoretical time complexity and model size for data selection.

Furthermore, to directly compare methods based on their performance gains and efficiency, we draw inspiration from Liu~\citep{liu2024take}, positing that methods with higher computational complexity should yield better data selection and, consequently, greater performance gains. To illustrate this trend, we apply linear regression to derive a regression line that represents the relationship between performance gains and runtime across all methods, as shown in Figure~\ref{tradeoff}. We then compute the perpendicular distances (represented by the dashed lines in Figure~\ref{tradeoff}) from each method to the regression line. A larger perpendicular distance above the regression line indicates a better trade-off between performance gains and efficiency.
\begin{figure}
    \centering  
    \adjustbox{center}{\includegraphics[width=0.5\linewidth, height=0.4\textwidth]{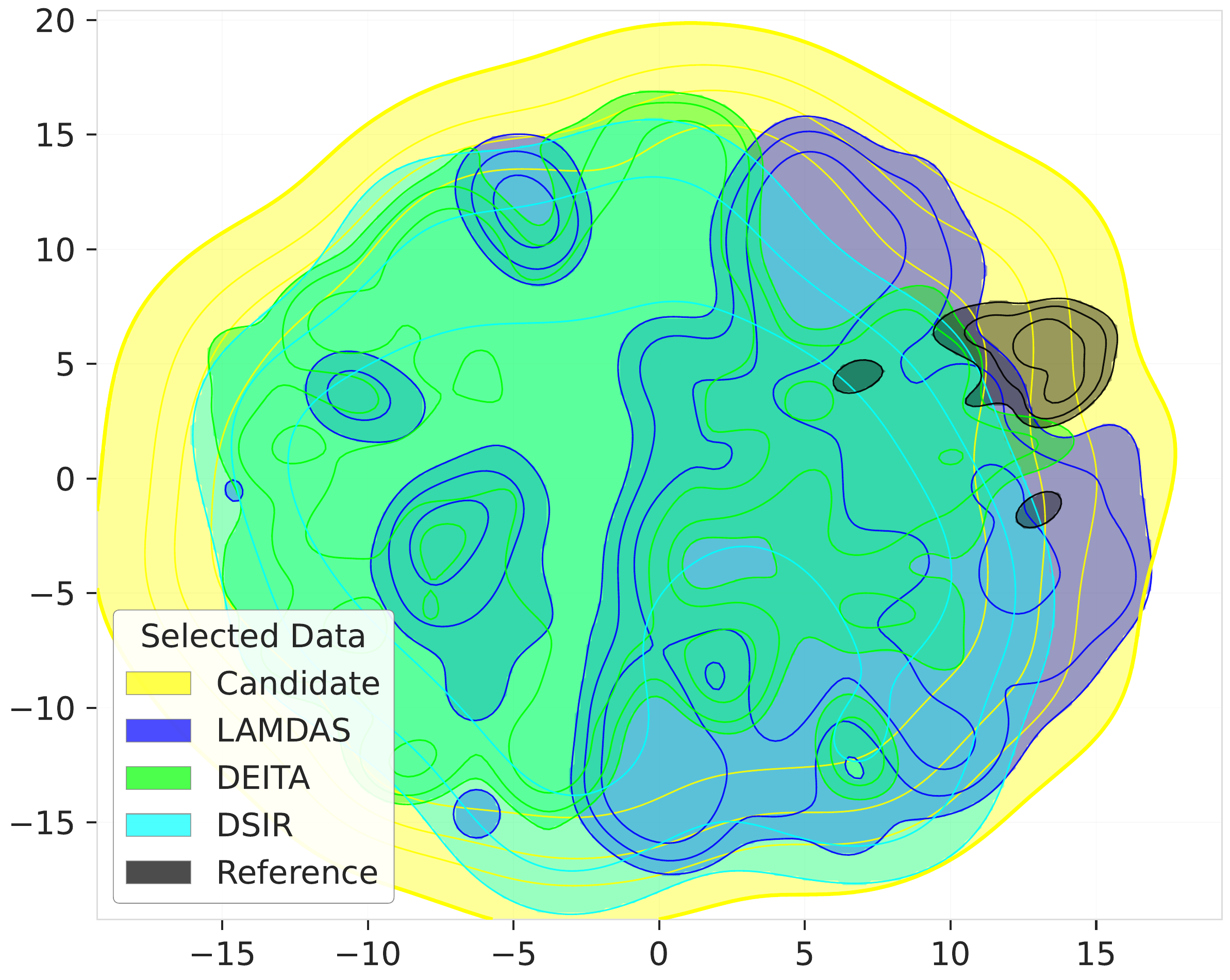}}
    \caption{Distribution comparison of candidate data (yellow), LAMDAS-selected data (blue), DEITA-selected data (green), DSIR-selected data (cyan), and reference data (black). Density contours represent the 10-50\% highest density regions for each distribution. Embeddings are normalized, and reduced to 2D space via t-SNE.}
    \label{distribution}
\end{figure}

Figure \ref{tradeoff} clearly illustrates that LAMDAS achieves a more favorable trade-off between performance gains and efficiency in both the CPT and SFT scenarios. While the two fastest methods, TSDS and DSIR, exhibit lower computational overhead, the data they select does not improve model performance to the same extent as the data selected by LAMDAS. Although the remaining methods may achieve better performance than TSDS and DSIR, they do so at the expense of efficiency. For example, LAMDAS can filter 100 billion tokens in 35 hours using only eight A100 GPUs, whereas LESS requires more than 144 hours for the same task. Moreover, the data selected by LESS fails to improve model performance as effectively as the data selected by LAMDAS.  Thus, LAMDAS surpasses these methods in terms of both performance gains and efficiency.

\subsection{Ablation Studies}

We also conduct a series of ablation studies to validate the impact of selection threshold $\tau$, prefix length, data selection model size, and data selection model type on the performance of LAMDAS. Due to the page limit, we summarize the major findings here and provide a detailed analysis in Appendix \ref{app:ablation}. The experimental results reveal four key insights: 1) the optimal value of $\tau = 1$ aligns with our design principle: to select candidate samples whose likelihood increases when conditioned on the domain prefix, thus ensuring relevance; 2) The prefix length of 30 represents an optimal "elbow point", balancing representation and efficiency; 3) the binary classification for data selection is a relatively simple task that can be effectively handled by small LLMs; 4) The domain specific models typically result in a more effective classifier.

\section{Conclusion}
In this paper, we propose LAMDAS, a novel domain-specific data selection method that utilizes a pre-trained LLM itself as an implicit classifier. By reframing data selection as a one-class classification (OCC) problem, LAMDAS circumvents the need for explicit feature engineering, common in similarity-based methods, and the computationally intensive retraining required by direct optimization methods. Experimental results demonstrate that LAMDAS achieves a superior balance between effectiveness and efficiency compared to nine SOTA baselines. Notably, models continually pre-trained on data selected by LAMDAS, using a 15\% selection ratio, outperform those trained on the full dataset by an average of at least 20\%. It is important to note that an ideal dataset typically exhibits high quality, broad coverage of the target domain, and high diversity, and may also introduce novel knowledge to the model, rather than simply containing samples already encountered during initial training. While various data selection methods have been proposed to address each of these aspects individually, LAMDAS can be readily integrated with these methods in a pipeline architecture to further enhance the domain-specific nature of the selected dataset. In future work, we aim to develop such pipelines for comprehensive data selection.

\bibliographystyle{colm2024_conference}
\bibliography{custom}

\appendix

\section{Experiment Setup}\label{sec:experiment_setup}

\begin{table*}[ht]
\small
\centering
\setlength{\tabcolsep}{1mm}
\begin{tabular}{llllllllll}
\toprule
Languages & Python  & Java    & JS      & C       & C++     & C\#     & PHP     & Others & Total \\
\midrule
Samples   & 24.45M & 53.03M & 30.67M & 14.19M & 18.47M & 22.72M & 24.44 M & 64.53M & 252.5M\\
\bottomrule
\end{tabular}
\caption{\label{cpt_statistics} The data statistics of CPT candidate data. We only list the languages that contain more than 10 million samples. The total dataset comprises 252.5 million samples and approximately 400 billion tokens.}
\end{table*}

\begin{table}[ht]
\centering
\small
\setlength{\tabcolsep}{0.8mm}
\begin{tabular}{llll}
\toprule
Domain                & Dataset          & Task                      & Samples \\
\midrule
\multirow{8}{*}{Code} & CODE-COMMENT     & code comment generation   & 645711  \\
                      & CODE-COMPLETION  & code completion           & 192,547 \\
                      & TEXT2CODE        & text to code generation   & 94086   \\
                      & CODE-TRANS       & code translation          & 307585  \\
                      & UNIT-TEST        & unit test-case generation & 390393  \\
                      & CODE-FEEDBACK    & code generation           & 66383   \\
                      & EVOL-INST        & code generation           & 66862   \\
                      & CODE-EXERCISE    & code generation           & 109827  \\
\midrule
MATH                  & OpenR1-MATH-220K & math reasoning            & 225129 \\
\bottomrule
\end{tabular}
\caption{\label{sft_statistics} Data statistics for candidate SFT data related to code and math tasks, including the specific tasks and data sizes.}
\end{table}

\subsection{Reference Data} The reference data is collected from diverse high-quality sources within each domain. We aim to present distinct, high-quality datasets that are representative of the types of problems tackled in each domain.
\begin{itemize}
\item \textbf{Code:}
LiveCodeBench2023 \citep{jain2025livecodebench} provides a diverse and contamination-free benchmark for evaluating LLMs' code-related capabilities. This benchmark features 400 high-quality programming problems sourced from LeetCode, AtCoder, and CodeForces competitions between May 2023 and March 2024, assessing multifaceted skills including code generation, self-repair, code execution, and test output prediction.
\item \textbf{Math:}
The math reference data consists of two key datasets focused on high-quality mathematical reasoning: LIMO \citep{Ye2025LIMOLI} and O1-journey \citep{Qin2024O1RJ}. LIMO is a carefully curated collection of 817 high-quality mathematical reasoning problems, each with detailed step-by-step solutions. Designed to challenge conventional scaling laws, LIMO demonstrates that models can achieve state-of-the-art performance on benchmarks like AIME24 and MATH-500 through quality-focused training, rather than relying solely on large-scale data accumulation. Complementing LIMO, O1-journey is a specialized training corpus developed by GAIR-NLP to implement the journey learning paradigm, featuring 327 structured reasoning traces with iterative problem-solving steps (including hypothesis testing, reflection, and backtracking). 
\end{itemize}



\subsection{Candidate Data} The candidate data, collected for CPT and SFT, includes:
\begin{itemize}
\item \textbf{Code CPT data}: This dataset is constructed by integrating code samples collected from GitHub and the open-source the-stack v1 corpus \footnote{https://huggingface.co/datasets/bigcode/the-stack}. To ensure data quality, we employ an automated filtering pipeline leveraging the Ask-LLM \citep{Sachdeva2024HowTT} methodology. Specifically, a Qwen-1.8B \footnote{https://huggingface.co/Qwen/Qwen-1\_8B} model, fine-tuned on annotated data generated by a larger 14B labeling model, is utilized to remove low-quality content. The resulting pre-training corpus comprises approximately 250 million code samples (400B tokens), spanning 30 programming languages including Python, Java, C++, C\#, and Go. All subsequent pre-training phases are conducted exclusively on this curated dataset.
\item \textbf{Code SFT data}: This dataset consists of 1,800,000 instruction-response pairs from high-quality and diverse benchmarks, such as Evol-Instruct~\citep{xu2024wizardlm} and CODE-FEEDBACK~\citep{zheng2024opencodeinterpreter}, which encompass code generation and debugging scenarios.
\item \textbf{Math SFT data}: This dataset comprises 220,000 samples from OpenR1-MATH-220K~\citep{DeepSeekAI2025DeepSeekR1IR}, focusing on mathematical reasoning and problem-solving.
\end{itemize}
The detailed data statistics of the CPT and SFT data are presented in Tables~\ref{cpt_statistics}
and~\ref{sft_statistics}.

\subsection{Evaluation Benchmarks}
For the coding task, the benchmarks include:
\begin{itemize}
\item \textbf{HumanEval (HE/HE+)}~\citep{chen2021codex}: A widely used benchmark consisting of 164 hand-crafted Python programming problems, requiring models to generate functionally correct code snippets from natural language descriptions. Each problem includes a function signature and a docstring.
\item \textbf{MBPP/MBPP+}~\citep{Austin2021ProgramSW}: An evaluation benchmark containing 1,000 crowd-sourced Python programming problems that cover diverse tasks, focusing on programming fundamentals and standard library functionality.
\item \textbf{LiveCodeBench}~\citep{jain2025livecodebench}: This evaluation benchmark contains 1,000 diverse Python problems, continuously collecting new problems from contests across LeetCode, AtCoder, and CodeForces. To prevent data contamination, we used LiveCodeBench data collected from January 2025 to April 2025 for testing purposes, which is distinct from the data used in the reference dataset (May 2023 to March 2024).
\item \textbf{CRUXEval}~\citep{gu2024cruxeval}: This benchmark comprises 800 Python functions and input-output pairs, split into two tasks: CRUXEval-I (input prediction) and CRUXEval-O (output prediction).
\end{itemize}

For the math reasoning task, the benchmarks include:
\begin{itemize}
\item \textbf{GSM8K}~\citep{cobbe2021gsm8k}: This collection of 8,500 linguistically diverse grade school math word problems, created by human writers and released by OpenAI, is designed to evaluate multi-step reasoning.
\item \textbf{MATH500}~\citep{lightman2023lets}: This benchmark measures advanced mathematical problem-solving ability.
\item \textbf{AIME2024}\footnote{https://huggingface.co/datasets/Maxwell-Jia/AIME\_2024}: This is a curated collection of 30 challenging mathematical problems from the 2024 American Invitational Mathematics Examination (AIME). These problems include detailed problem statements, step-by-step solutions, and numerical answers, designed to evaluate LLMs' advanced reasoning capabilities in domains like algebra, geometry, and number theory.
\item \textbf{AMC23}\footnote{https://huggingface.co/datasets/zwhe99/amc23}: This comprises 49 problems from the 2023 American Mathematics Competitions (AMC), featuring challenging questions in algebra, geometry, number theory, and combinatorics designed for high school students. It serves as a valuable resource for educators, competition preparation, and AI research focused on mathematical reasoning and problem-solving.
\end{itemize}
Note that we utilize evaluation benchmarks that are relevant to the domain but extend beyond the specific tasks found in the reference data. This allows us to determine if the model, trained on the selected data, can effectively generalize to related tasks within the same field.

\subsection{Baselines}
To evaluate the effectiveness of LAMDAS in the CPT stage, we compare its performance against seven SOTA data selection methods:

\begin{table*}[!t]
\centering
\small


\begin{tabular}{p{1.0\linewidth}}
\toprule
\textbf{\textit{Code:}} "from setuptools import setup CLASSIFIERS = [ Programming Language :: Python :: 3.8, Programming Language :: Python :: 3.9, ] setup( name="pyLdtk", version=0.0.3, description="Just some quick and simple groundwork for working with ldtk in python. VERY EARLY STAGES!", long\_description=description, url="https://github.com/LGgameLAB/pyLdtk", author="Luke Gonsalves", author\_email="lukegonsalves07@gmail.com", license="MIT", packages=["pyLdtk"], classifiers=CLASSIFIERS, install\_requires=REQUIREMENTS, keywords="ldtk LDTK python pygame py Tiled tile tileset game")" \\
\textbf{\textit{[Likelihood Ratio]:}} 0.78 \\
\textbf{\textit{[PPL:]}} 3.78. \\ 
\midrule
\textbf{\textit{Code:}} "from \_\_future\_\_ import absolute\_import
from highway\_env.envs.highway\_env import HighwayEnv
from highway\_env.envs.highway\_env\_continuous import HighwayEnvCon from highway\_env.envs.highway\_env\_continuous\_intrinsic\_rew import HighwayEnvCon\_intrinsic\_rew 
from highway\_env.envs.merge\_env import MergeEnv
from highway\_env.envs.roundabout\_env import RoundaboutEnv from highway\_env.envs.parking\_env import ParkingEnv from highway\_env.envs.two\_way\_env import TwoWayEnv from highway\_env.envs.roundabout\_env import RoundaboutEnv from highway\_env.envs.parking\_env import ParkingEnv from highway\_env.envs.two\_way\_env import\ TwoWayEnv from highway\_env.envs.roundabout\_env import RoundaboutEnv from highway\_env.envs.parking\_env import ParkingEnv from highway\_env.envs.two\_way\_env import TwoWayEnv" \\
\textbf{\textit{[Likelihood Ratio]:}} 0.92 \\
\textbf{\textit{[PPL:]}} 2.34 \\ 
\midrule
\textbf{\textit{Code:}} "import os import time os.system("qsub enc\_1") time.sleep(30) os.system("qsub enc\_2") time.sleep(30) os.system("qsub enc\_3") time.sleep(30) os.system("qsub enc\_4") time.sleep(30) os.system("qsub enc\_5") time.sleep(30) os.system("qsub enc\_6") time.sleep(30) os.system("qsub enc\_7") time.sleep(30) os.system("qsub enc\_8") time.sleep(30) os.system("qsub enc\_9") time.sleep(30) os.system("qsub enc\_10") time.sleep(30) os.system("qsub enc\_11") time.sleep(30) os.system("qsub enc\_12") time.sleep(30) os.system(\"qsub enc\_13\") time.sleep(30) os.system("qsub enc\_14") time.sleep(30) os.system("qsub enc\_15") time.sleep(30) os.system("qsub enc\_16") time.sleep(30) os.system("qsub enc\_17") time.sleep(30) os.system("qsub enc\_18") time.sleep(30) os.system("qsub enc\_19") time.sleep(30) os.system("qsub enc\_20") time.sleep(30) os.system("qsub enc\_21") time.sleep(30) os.system("qsub enc\_22") time.sleep(30) os.system("qsub enc\_23") time.sleep(30) os.system("qsub enc\_24") time.sleep(30) os.system("qsub enc\_25") time.sleep(30) os.system("qsub enc\_26") time.sleep(30) os.system("qsub enc\_27")" \\
\textbf{\textit{[Likelihood Ratio]:}} 0.55 \\
\textbf{\textit{[PPL:]}} 1.28 \\ 
\midrule
\textbf{\textit{Code:}} "def add\_native\_methods(clazz): def initIDs\_\_java\_lang\_Class\_\_java\_lang\_Class\_\_java\_
lang\_Class\_\_java\_lang\_Class\_\_java\_lang\_Class\_\_java\_lang\_Class\_\_java\_lang\_Class\_\_java\_lang\_Class\_\_java\_
lang\_Class\_\_java\_lang\_Class\_\_(a0, a1, a2, a3, a4, a5, a6, a7, a8, a9, a10, a11): raise NotImplementedError() def registerNativeLoops\_\_\_\_(a0): raise NotImplementedError() clazz.initIDs\_\_java\_lang\_Class\_\_java\_lang\_Class\_\_java\_lang\_Class\_\_java\_lang\_Class\_\_java\_lang\_Class
\_\_java\_lang\_Class\_\_java\_lang\_Class\_\_java\_lang\_Class\_\_java\_lang\_Class\_\_java\_lang\_Class\_\_java\_lang\_Class\_\_
=staticmethod(initIDs\_\_java\_lang\_Class\_\_java\_lang\_Class\_\_java\_lang\_Class\_\_java\_lang\_Class\_\_java\_lang\_
Class\_\_java\_lang\_Class\_\_java\_lang\_Class\_\_java\_lang\_Class\_\_java\_lang\_Class\_\_java\_lang\_Class\_\_java\_lang\_
Class\_\_) clazz.registerNativeLoops\_\_\_\_ = staticmethod(registerNativeLoops\_\_\_\_)" \\
\textbf{\textit{[Likelihood Ratio]:}} 0.57 \\
\textbf{\textit{[PPL:]}} 2.30 \\ 
\midrule
\end{tabular}
\caption{\label{case_study} Samples selected by the "PPL" method but filtered out by LAMDAS in the CPT experiment. A lower PPL value typically indicates that the next tokens are easier to predict, often due to the presence of repetitive strings or patterns, or that the corresponding sample frequently appears in the pretraining corpus. As a result, this method may inadvertently select samples that do not enhance overall quality.}
\end{table*}

\begin{itemize}
\item \textbf{DEITA}~\citep{liu2024what}: Evol-instruct is employed to generate data with different quality and complexity. ChatGPT is then exploited to score the quality and complexity of the synthesized data. Finally, two LLM-based scorers, trained on the ChatGPT-scored data, assess the quality and complexity of each sample within the candidate set. For CPT, we use the scorers provided in XCoder \citep{wang-etal-2024-code}. For SFT, we train two separate scorers based on the Qwen2.5-0.5B model.
\item \textbf{DsDm}~\citep{10.5555/3692070.3692568}: This method implements the TRAK framework \citep{park2023trak} to estimate the optimal weight for each data point in the selection process. TRAK employs a closed-form expression that requires first training $m = 4$ reference models on the reference dataset, then collecting projected gradients for each model, and finally substituting these gradients into the closed-form expression.
\item \textbf{DSIR}~\citep{xie2023data}: This method uses importance resampling to select data points that better match the reference distribution, based on n-gram features and the KL reduction metric.
\item \textbf{TSDS}~\citep{liu2024tsds}: This method optimizes task-specific data selection by recasting it as an optimal transport problem using text embeddings as features. It further incorporates a diversity-promoting regularizer, based on kernel density estimation, to mitigate near-duplicate effects.
\item \textbf{Perplexity-based Filtering (PPL)}: This method filters data based on perplexity scores, computed using a Qwen2.5-Coder-7B model fine-tuned on our reference data, selecting samples with lower perplexity values. We also present some examples given by this method in Table~\ref{case_study}.
\item \textbf{Random Sampling (Random)}: This method selects data points randomly from the candidate dataset.
\item \textbf{Full Data Training (Full)}: This method uses the entire candidate dataset for training without any selection.
\end{itemize}

\subsection{Implementation Details of LAMDAS} The prefix tuning is performed on a single NVIDIA A100 GPU (80GB) using the AdamW optimizer with a learning rate of 1e-3 and a weight decay of 0.1. This process runs for 10 epochs with a batch size of 4, consuming approximately 1 hour for the domain-specific reference datasets. For data selection, we utilize eight A100 GPUs with the same batch size of 4. This setup allows us to process 100 billion tokens in 35 hours, equating to about 5 seconds per 100 samples. During the CPT stage, we employ 64 A100 GPUs, using a batch size of 8 per device, a learning rate of 1e-4, and a weight decay of 0.1. Training runs for 4 epochs until early stopping, with a dataset of 15 billion selected tokens. For the SFT stage, we utilize 32 A100 GPUs with a batch size of 8 per device, a learning rate of 1e-5, and a weight decay of 0.1. Training again runs for 4 epochs until early stopping, using a dataset comprising 750,000 code samples and 70,000 math samples.

\begin{table*}[t]
\centering
\small

\centering
\resizebox{1.0\textwidth}{!}{
\begin{tabular}{cccccccc}
\toprule
Methods &Data Size  & Models & MATH500 & AIME  & GSM8K & AMC & AVG\\
\midrule
\multirow{4}{*}{LAMDAS(ours)} & \multirow{4}{*}{70K} & Qwen2.5-0.5B-Instruct & { {\textbf{15.2}}} & { {\textbf{6.7}}} & { {\textbf{17.3}}} & {{\textbf{12.5}}}& { {\textbf{12.9}}}\textit{\scriptsize{+18.3\%}} \\

& & Qwen2.5-1.5B-Instruct & { {\textbf{18.4}}}& { {\textbf{13.3}}} & { {\textbf{27.9}}}& { {\textbf{24.5}}}& { {\textbf{21.0}}}\textit{\scriptsize{+22.1\%}}\\

& & Qwen2.5-7B-Instruct   & { {\textbf{29.7}}} & { {\textbf{26.7}}}& { {\textbf{39.8}}}& { {\textbf{38.8}}}& { {\textbf{33.8}}}\textit{\scriptsize{+24.3\%}}\\
& & Qwen2.5-32B-Instruct & { {\textbf{43.1}}} & { {\textbf{33.3}}} & { {\textbf{62.7}}} & { {\textbf{57.1}}} & { {\textbf{49.1}}}\textit{\scriptsize{+9.3\%}}\\

\midrule
\multirow{4}{*}{DSIR \citep{xie2023data}} & \multirow{4}{*}{70K} & Qwen2.5-0.5B-Instruct & 10.2 & 3.3 & 14.3 & 6.1 & 8.4\textit{\scriptsize{-22.9\%}}\\
& & Qwen2.5-1.5B-Instruct & 16.8 & 3.3 & 23.1 & 22.5 &16.4\textit{\scriptsize{-4.7\%}}\\
& & Qwen2.5-7B-Instruct & 28.5 & 13.3 & 36.2 & 26.5&26.1\textit{\scriptsize{-4.0\%}}\\
& & Qwen2.5-32B-Instruct & 38.1 & 30.0 & 58.2 & 53.0 & 44.8\textit{\scriptsize{-0.1\%}} \\

\midrule
\multirow{4}{*}{TSDS \citep{liu2024tsds}} &\multirow{4}{*}{70K}& Qwen2.5-0.5B-Instruct & 10.7 & 0.0 & 15.2 & 6.1 &8.0\textit{\scriptsize{-26.6\%}}\\
& & Qwen2.5-1.5B-Instruct & 16.1 & 3.3 & 23.7 & 15.9&14.8\textit{\scriptsize{-14.1\%}}\\
&  & Qwen2.5-7B-Instruct & 28.7 & 13.3 & 37.4 & 26.5 &26.4\textit{\scriptsize{-2.9\%}}\\
& & Qwen2.5-32B-Instruct & 38.1 & 30.0 & 58.9 & 55.1 & 45.5\textit{\scriptsize{+1.8\%}} \\

\midrule
\multirow{4}{*}{DEITA \citep{liu2024what}}& \multirow{4}{*}{70K}  & Qwen2.5-0.5B-Instruct & 12.9 & 3.3 & \underline{16.7} & 8.1 & 10.3\textit{\scriptsize{-5.5\%}}\\
& & Qwen2.5-1.5B-Instruct & 17.6 & 6.7 & 24.5 & 18.4 & 16.8\textit{\scriptsize{-2.3\%}}\\
 &  & Qwen2.5-7B-Instruct & 28.7 & 20.0 & 37.8 & 30.6& 29.3\textit{\scriptsize{+7.7\%}}\\
& & Qwen2.5-32B-Instruct & 40.6 & 26.7 & 59.8 & 55.1 & 45.6\textit{\scriptsize{+1.6\%}} \\

\midrule
\multirow{4}{*}{DsDm \citep{10.5555/3692070.3692568}}& \multirow{4}{*}{70K}  & Qwen2.5-0.5B-Instruct & 13.1 & 3.3 & 16.5 & 9.8 & 10.7\textit{\scriptsize{-1.8\%}} \\
& & Qwen2.5-1.5B-Instruct & \underline{18.1} & 6.7 & \underline{24.7} & \underline{21.6} & \underline{17.8}\textit{\scriptsize{+3.5\%}} \\
 & & Qwen2.5-7B-Instruct & \underline{28.5} & 16.7 & 37.1 & \underline{33.8} & 29.0\textit{\scriptsize{+6.6\%}} \\
& & Qwen2.5-32B-Instruct & 40.2 & 26.7 & 59.7& \underline{56.2} & 45.7\textit{\scriptsize{+1.8\%}}\\

\midrule
\multirow{4}{*}{Nuggets \citep{Li2023OneSL}}& \multirow{4}{*}{70K}  & Qwen2.5-0.5B-Instruct & \uwave{13.7} & \underline{3.3} & 16.5 & \underline{10.2} & \underline{10.9}\textit{\scriptsize{+0.0\%}}\\

& & Qwen2.5-1.5B-Instruct & \underline{17.8} & \underline{6.7} & 24.2 & 20.4 & 17.4\textit{\scriptsize{+1.2\%}}\\

 &  & Qwen2.5-7B-Instruct & 27.8 & \underline{20.0} & \underline{37.7} & 32.7 &\underline{29.6}\textit{\scriptsize{+8.8\%}} \\
& & Qwen2.5-32B-Instruct &\underline{40.9} & \underline{26.7} & \underline{60.3} & 55.1 & \underline{45.8}\textit{\scriptsize{+2.0\%}} \\

\midrule
\multirow{4}{*}{LESS \citep{xia2024less}}& \multirow{4}{*}{70K}  & Qwen2.5-0.5B-Instruct & \underline{13.5} & \uwave{3.3} & \uwave{16.9} & \uwave{10.2} & \uwave{11.0}\textit{\scriptsize{+0.9\%}}\\

& & Qwen2.5-1.5B-Instruct & \uwave{18.1} & \uwave{10.0} & \uwave{24.8} & \uwave{22.4} & \uwave{18.8}\textit{\scriptsize{+9.3\%}}\\

 &  & Qwen2.5-7B-Instruct & \uwave{29.1} & \uwave{20.0} & \uwave{38.5} & \uwave{36.7} & \uwave{31.1}\textit{\scriptsize{+14.3\%}}\\
 
& & Qwen2.5-32B-Instruct & \uwave{41.2} & \uwave{30.0} & \uwave{60.8} & \uwave{57.1} & \uwave{47.3}\textit{\scriptsize{+5.3\%}}\\

\midrule
\multirow{4}{*}{Random} & \multirow{4}{*}{70K} & Qwen2.5-0.5B-Instruct & 8.6 & 0.0 & 13.3 & 6.1 &7.0\textit{\scriptsize{-35.7\%}}\\
& & Qwen2.5-1.5B-Instruct & 14.2 & 3.3 & 21.7 & 12.2 & 12.9\textit{\scriptsize{-25.0\%}}\\
& & Qwen2.5-7B-Instruct & 26.6 & 16.7 & 34.1 & 30.6 & 27.0\textit{\scriptsize{-0.7\%}}\\
& & Qwen2.5-32B-Instruct & 35.4 & 26.7 & 54.2 & 40.8 & 39.3\textit{\scriptsize{-14.2\%}}\\

\midrule
\multirow{4}{*}{Full} & \multirow{4}{*}{220K}  & Qwen2.5-0.5B-Instruct & 14.1 & 3.3 & 17.9 & 8.1 & 10.9\\
& & Qwen2.5-1.5B-Instruct & 17.4 & 6.6 & 26.3 & 18.4 & 17.2\\
& & Qwen2.5-7B-Instruct & 26.2 & 16.7 & 35.3 & 30.6 & 27.2\\    
& & Qwen2.5-32B-Instruct & 37.7 & 33.3 & 57.7 & 51.0 & 44.9 \\

\midrule
\multirow{4}{*}{Base}& \multirow{4}{*}{220K}  & Qwen2.5-0.5B-Instruct & 5.4 & 0.0 & 10.7 & 0.0 & 4.0 \\
& & Qwen2.5-1.5B-Instruct & 7.8 & 0.0 & 15.2 & 4.1 & 6.3\\
 & & Qwen2.5-7B-Instruct & 20.3& 6.7 & 27.8 & 8.2 & 15.8 \\  
& & Qwen2.5-32B-Instruct & 34.8 & 16.7 & 45.2 & 14.3 & 27.8\\
\bottomrule
\end{tabular}}
\caption{\label{SFT_math_results}Performance of Qwen2.5 serious models trained on various data selection methods on MATH500, AIME, GSM8K, and AMC. We also compute the average of models' performance on all benchmarks and calculate the proportion of improvement compared with the "Full" data. The best results are in {\textbf{bold}}, the second are \uwave{underlined}, and the third are \underline{highlighted}.}
\end{table*}

\section{Math} 
\label{app:math}

This section examines the application of data selection in mathematical reasoning tasks. As highlighted in LIMO~\citep{Ye2025LIMOLI} and O1-journey~\citep{Qin2024O1RJ}, creating high-quality reasoning data for SFT is a challenging process that often involves complex and extensive pipelines. Given the difficulty in curating such data, employing data selection techniques to identify similar high-quality data becomes crucial. By leveraging the curated data in LIMO and O1-journey as a reference set, we conduct data selection using various methods and evaluated the resulting models on benchmarks designed to assess mathematical reasoning abilities. As shown in Table~\ref{SFT_math_results}, our proposed method consistently achieves the best performance, with the Qwen2.5-32B-Instruct model trained on data selected by LAMDAS outperforming the same model trained on the "Full" dataset by an average of 9.3\%. This demonstrates LAMDAS's broad applicability across various domains and tasks. 

The performance trends for the baseline methods align with those observed in coding tasks, and therefore, their detailed discussion is omitted here. However, unlike the coding tasks, we find that the data selected by some methods does not improve performance as much as the full dataset, particularly when training smaller LLMs. This may be because the original candidate data is already of high quality and relevant to the math reasoning task, leading to potential performance drops if the data is not carefully selected. This further underscores the effectiveness of LAMDAS, which can enhance performance using less than 30\% of the full dataset. Moreover, larger models benefit from extensive pretraining, allowing them to activate vast amounts of knowledge when fine-tuned with a small selection of data, thereby improving performance on downstream tasks without risking overfitting~\citep{chen2023maybe, zhou2023lima}. In contrast, smaller models have limited learning capacities. Fine-tuning on a larger, diverse dataset enables them to absorb more information and achieve better overall performance, even with some irrelevant data. This explains the more significant performance drop observed in smaller models trained with selected data. Notably, data selected by LAMDAS, as well as by LESS and Nuggets, achieves performance gains over the full dataset regardless of model size, highlighting the robustness of these methods.

 \begin{figure}[t]
    \centering
    \adjustbox{center}{\includegraphics[width=0.6\linewidth, height=0.375\textwidth]{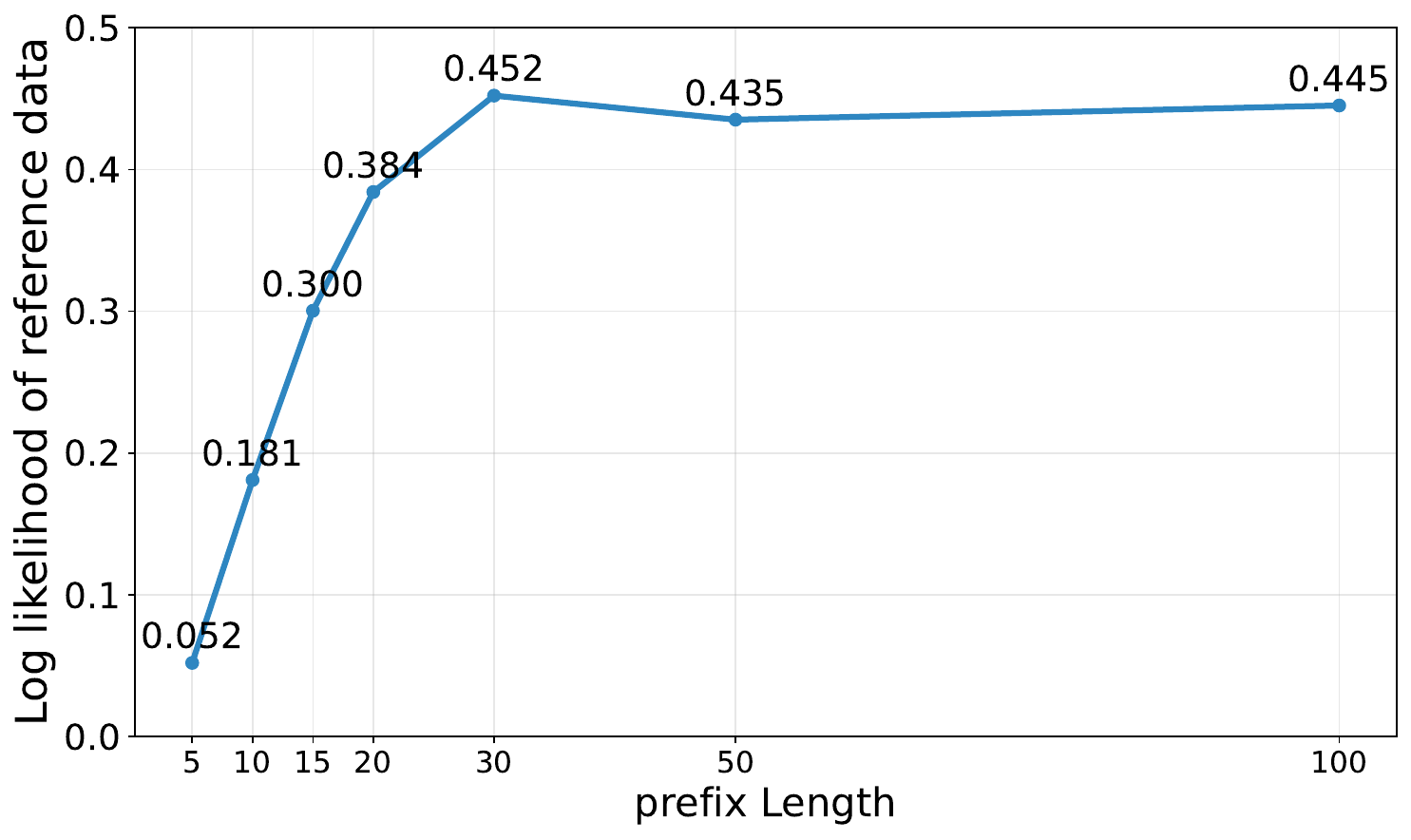}}
    \caption{Impact of the prefix lengths on the log-likelihood of the reference data. The results indicate that a prefix length of 30 yields the highest log-likelihood.}
    \label{prefix_length}
\end{figure}

\begin{figure*}
    \centering
    \adjustbox{center}{\includegraphics[width=1.05\linewidth]{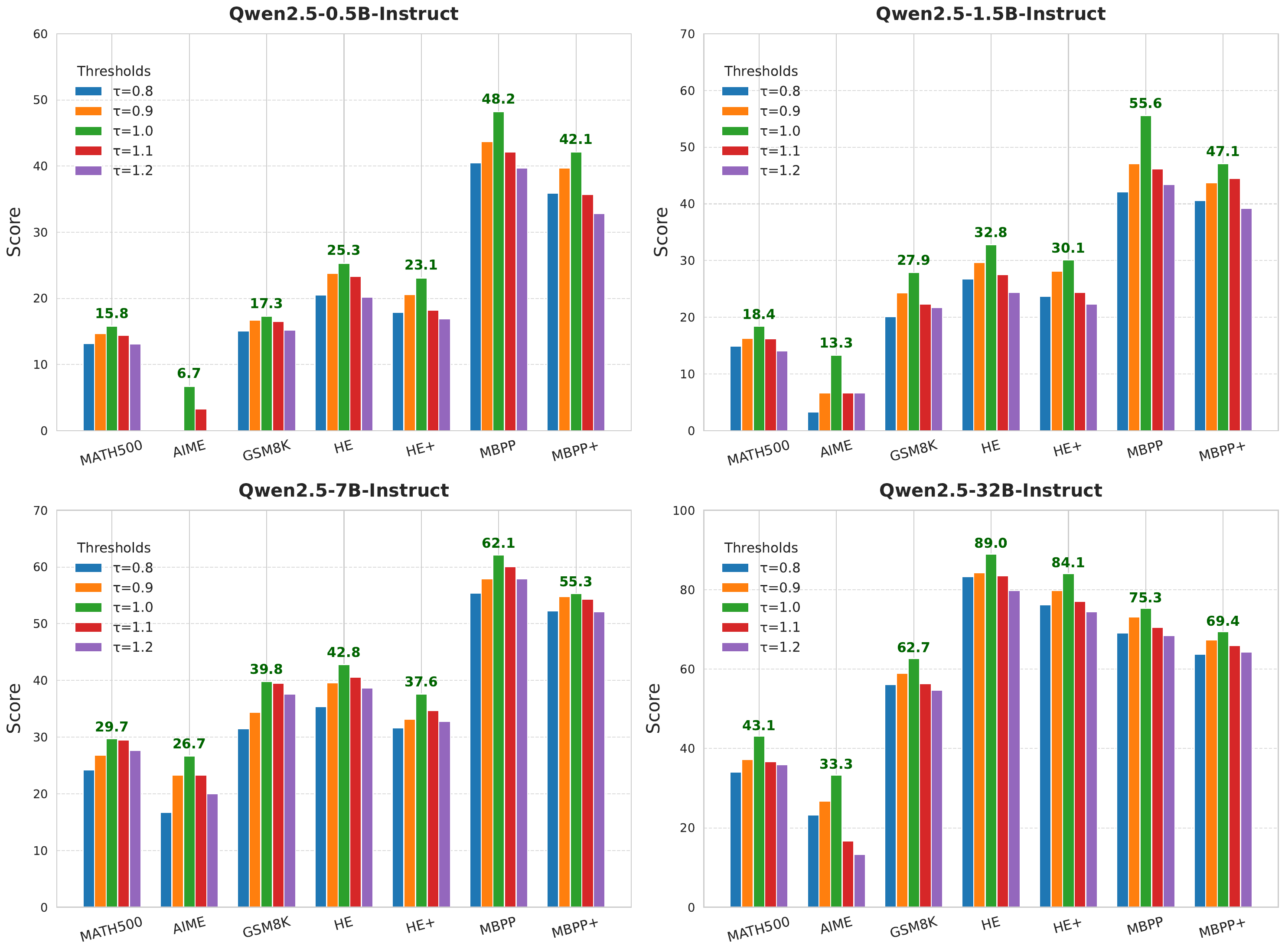}}
    \caption{\label{ablation}Ablation Studies on the impact of varying thresholds on the selected SFT data resulting from LAMDAS. Each threshold results in different subset sizes. Specifically, at thresholds of 0.8, 0.9, 1.0, 1.1, and 1.2, the corresponding data retention ratios are approximately 50\%, 40\%, 30\%, 10\%, and 5\%.}
\end{figure*}

\begin{table*}[t]
\centering
\small

\begin{tabular}{lllllllll}
\toprule
Data Selection Models      & Data Size  & Speed s/k & HE   & HE+  & MBPP & MBPP+ & LCB & CRUXEval \\
\midrule
Qwen2.5-Coder-0.5B & 750K & 38s & 25.3 & 23.1 & 48.2 & 42.1  & 8.7 & 15.1 \\
Qwen2.5-Coder-3B   & 739K & 65s & 25.8 & 23.5 & 48.5 & 42.7 & 8.8 & 15.3 \\
Qwen2.5-0.5B   & 842K  &34s & 21.1 & 18.7 & 44.3 & 41.6  & 7.1 & 13.5 \\
\bottomrule
\end{tabular}
\caption{\label{model_size_type}Impact of reference model size. The selected data are used to train the Qwen2.5-0.5B-Instruct model.}
\end{table*}

\section{Ablation Studies}
\label{app:ablation}

\subsection{Impact of the selection threshold $\tau$} As shown in Figure~\ref{ablation}, a high threshold over-filters valuable examples, reducing diversity and harming generalization, while a low threshold introduces noisy data, degrading performance. An optimal threshold of 1.0 balances quality and diversity, leading to peak performance across tasks regardless of the underlying model for SFT. 

\subsection{Impact of the prefix length} Figure \ref{prefix_length} illustrates the relationship between prefix length and the log likelihood on reference data. The results indicate that a prefix length of approximately 30 maximizes this likelihood. Shorter prefix lengths are insufficient to accurately capture the characteristics of the reference data. In contrast, lengths exceeding 30 do not provide any further significant increase in data likelihood.

\subsection{Impact of the model size} As presented in Table~\ref{model_size_type}, using the Qwen2.5-Coder models with 0.5B and 3B parameters as the implicit classifier produces a comparable performance, with only marginal improvements observed with the larger model. Leveraging a smaller LLM for classification, in turn, further enhances the efficiency of LAMDAS.

\subsection{Impact of the model type} We investigate the impact of utilizing LLMs pre-trained on datasets from different domains as classifiers. The results presented in Table~\ref{model_size_type} show that employing code LLMs for data selection within coding domains enhances performance. This improvement can be attributed to the specialized knowledge these pretrained models possess, making them more effective classifiers for tasks in the coding domain.

\clearpage

\end{document}

%% file: math_commands.tex
\usepackage{amsmath,amsfonts,bm}

\def\1{\bm{1}}

\def\vx{{\bm{x}}}
\def\vy{{\bm{y}}}

\def\mC{{\bm{C}}}

\DeclareMathAlphabet{\mathsfit}{\encodingdefault}{\sfdefault}{m}{sl}
\SetMathAlphabet{\mathsfit}{bold}{\encodingdefault}{\sfdefault}{bx}{n}

\def\gD{{\mathcal{D}}}

\DeclareMathOperator*{\argmax}{arg\,max}